%% file: root.tex
\documentclass[letterpaper, 10 pt, conference]{ieeeconf}  

\IEEEoverridecommandlockouts                              

\usepackage{graphics} 
\usepackage{wrapfig}

\usepackage{amsthm} 
\usepackage{amsmath,amssymb} 

\usepackage{graphicx}
\usepackage{tabularx}
\usepackage{multirow, makecell}
\usepackage{diagbox}
\usepackage{slashbox}
\usepackage{rotating}
\usepackage{subcaption}
\usepackage{cite}
\usepackage{romannum}
\usepackage{url}
\usepackage{bm}
\usepackage[table]{xcolor}
\usepackage{hyperref}
\usepackage{siunitx} 

\usepackage{booktabs}
\usepackage{mathtools} 

\usepackage[ruled,vlined]{algorithm2e}

\newcommand{\revsecond}[1]{\textcolor{black}{#1}}

\newcommand{\vx}{{\boldsymbol x}}
\newcommand{\vu}{{\boldsymbol u}}
\newcommand{\vr}{{\boldsymbol r}}
\newcommand{\vp}{{\boldsymbol p}}
\newcommand{\lieder}{L}
\newcommand{\StateSpace}{\mathcal{X}}
\newcommand{\ControlSpace}{\mathcal{U}}
\newcommand{\RealSpace}{\mathbb{R}}
\newcommand{\Rn}{\mathbb{R}^{n}}
\newcommand{\Rm}{\mathbb{R}^{m}}
\newcommand{\Rnm}{\mathbb{R}^{n \times m}}
\newcommand{\calC}{\mathcal{C}} 
\newcommand{\calK}{\mathcal{K}} 

\newcommand{\calE}{\mathcal{E}} 
\newcommand{\calV}{\mathcal{V}} 
\newcommand{\calX}{\mathcal{X}} 
\newcommand{\calU}{\mathcal{U}} 
\newcommand{\calP}{\mathcal{P}} 
\newcommand{\calW}{\mathcal{W}} 
\newcommand{\calB}{\mathcal{B}}
\newcommand{\calO}{\mathcal{O}} 
\newcommand{\calH}{\mathcal{H}} 
\newcommand{\calS}{\mathcal{S}} 
\newcommand{\calF}{\mathcal{F}} 

\newtheorem{definition}{Definition}
\newtheorem{theorem}{Theorem}

\theoremstyle{definition}
\newtheorem{problem}{Problem}

\makeatother

\DeclareCaptionFont{mysize}{\fontsize{8}{9.6}\selectfont}
\title{\LARGE \bf
Visibility-Aware RRT* for Safety-Critical Navigation of Perception-Limited Robots in Unknown Environments}

\author{Taekyung Kim and Dimitra Panagou
\thanks{Taekyung Kim is with the Department of Robotics, University of Michigan, Ann Arbor, MI, 48109, USA {\tt\footnotesize taekyung@umich.edu}}
\thanks{Dimitra Panagou is with the Department of Robotics and Department of Aerospace Engineering, University of Michigan, Ann Arbor, MI, 48109, USA {\tt\footnotesize dpanagou@umich.edu}}
\thanks{This work has been partially supported by Amazon and partially by NSF Award No. 1942907.}%
}

\begin{document}
\maketitle
\thispagestyle{empty}
\pagestyle{empty}

\begin{abstract}
Safe autonomous navigation in unknown environments remains a critical challenge for robots with limited sensing capabilities. While safety-critical control techniques, such as Control Barrier Functions (CBFs), have been proposed to ensure safety, their effectiveness relies on the assumption that the robot has complete knowledge of its surroundings. In reality, robots often operate with restricted field-of-view and finite sensing range, which can lead to collisions with unknown obstacles if the planner is agnostic to these limitations. To address this issue, we introduce the Visibility-Aware RRT* algorithm that combines sampling-based planning with CBFs to generate safe and efficient global reference paths in partially unknown environments. The algorithm incorporates a collision avoidance CBF and a novel visibility CBF, which guarantees that the robot remains within locally collision-free regions, enabling timely detection and avoidance of unknown obstacles. We conduct extensive experiments interfacing the path planners with two different safety-critical controllers, wherein our method outperforms all other compared baselines across both safety and efficiency aspects. \href{https://www.taekyung.me/visibility-rrt}{\textcolor{red}{[Project Page]}}\footnote{Project page: \href{https://www.taekyung.me/visibility-rrt}{https://www.taekyung.me/visibility-rrt}} \href{https://github.com/tkkim-robot/visibility-rrt}{\textcolor{red}{[Code]}} \href{https://youtu.be/l4vlMTf8s74}{\textcolor{red}{[Video]}}

\end{abstract}

\section{INTRODUCTION}

\input{_I.Introduction/intro}

\section{PRELIMINARIES}
\subsubsection{System Description}
\input{_II.Preliminaries/a_dynamics}
\subsubsection{Control Barrier Functions}
\input{_II.Preliminaries/c_cbf}
\subsubsection{High Order CBF}
\input{_II.Preliminaries/d_hocbf}

\section{PROBLEM FORMULATION \label{sec:problem}}
\input{_III.Problem/_intro}

\section{GLOBAL PATH PLANNER \label{sec:planner}}

\input{_IV.Methodology/_intro}
\subsection{Visibility-Aware RRT* \label{subsec:lqr_cbf_rrtstar}}

\input{_IV.Methodology/a_visibility_rrt}
\subsection{HOCBF for Collision Avoidance \label{subsec:cbf-collision-avoidance}}
\input{_IV.Methodology/b_cbf}
\subsection{Incorporating Visibility Constraint into CBF \label{subsec:cbf-visibility}}

\input{_IV.Methodology/c_visibility}

\section{LOCAL TRACKING CONTROLLER \label{sec:controller}}
\input{_V.Local_Tracking_Controller/a_controller}

\section{RESULTS \label{sec:experiments}}

\input{_VI.Experiments/a_results}

\section{CONCLUSION}
\input{_VII.Conclusion/conclusion}

\addtolength{\textheight}{0 cm}   





\bibliographystyle{IEEEtran}
\typeout{}
\bibliography{references.bib}

\end{document}

%% file: _I.Introduction/intro.tex
Autonomous navigation in unknown environments is a fundamental challenge in robotics, with applications ranging from exploration and search-and-rescue to autonomous driving. A general navigation stack comprises four key components: localization, mapping, planning, and control. Localization and mapping enable robots to build a map of their surroundings using onboard sensors, such as Visual Simultaneous Localization and Mapping (VSLAM) systems. However, these sensors often have limitations, including restricted field-of-view (FOV) and finite sensing range. For instance, popular off-the-shelf RGB-D cameras' horizontal FOV are often less than 70$^\circ$ \cite{da_silva_neto_comparison_2020}.

\begin{figure}[t]
    \centering
    \includegraphics[width=0.91\linewidth]{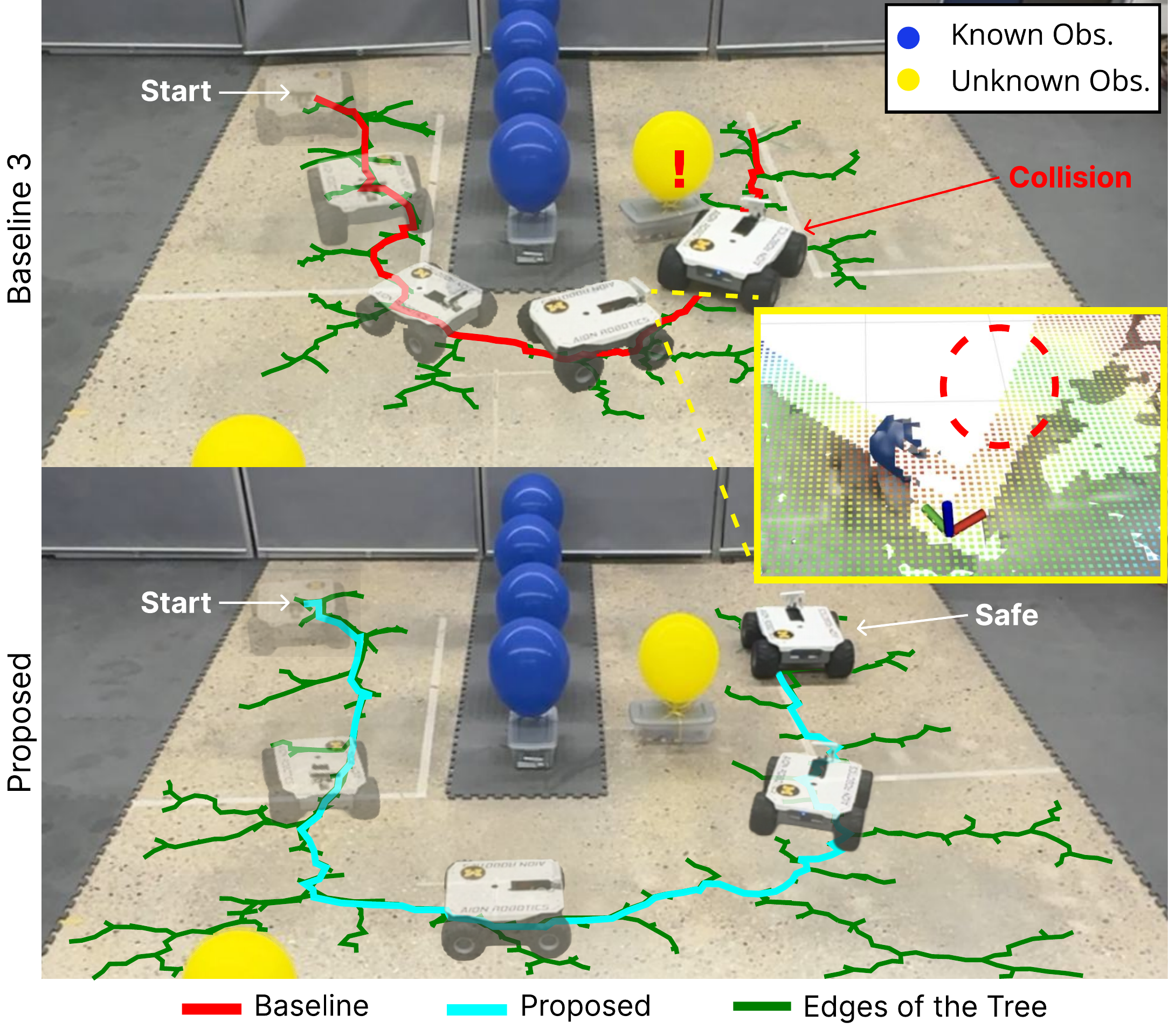}
    \caption{\revsecond{In this paper we develop an RRT*-based planner that accounts for sensing limitations, and in conjunction with a tracking controller, guarantees avoidance of \textit{a priori} unknown obstacles. This figure shows trajectories generated by baseline (\textbf{Top}) and visibility-aware (\textbf{Bottom, Ours}) planning methods, tracked by the rover using a CBF-QP controller (baselines and results presented in Section \ref{sec:experiments}.). \textbf{Inset: } Snapshot of the Signed Distance Field map just before the collision while following the baseline trajectory, showing that the unknown obstacle (which is located in the dashed circle, but not in the FOV of the robot) is not detected in time, leading to collision.}}
    \label{fig:main}
\end{figure}

To ensure safe navigation, there has been growing interest in safety-critical controllers, such as those based on Control Barrier Functions~(CBFs)~\cite{ames_control_2019}, which enforce safety objectives on controlled systems. These controllers theoretically guarantee safety by keeping the robot within a safe set, assuming that nearby obstacles are fully known \textit{a priori}. However, in practice, the effectiveness of these controllers relies heavily on the robot's ability to perceive obstacles in the local environment. If the global planner does not account for the robot's limited sensing capabilities under the presence of unknown obstacles, the theoretical safety guarantees may not hold, leading to possible safety violations.

Several approaches have been proposed to improve the safety of navigation under sensing limitations. Some methods use nonlinear optimization to generate~\cite{spasojevic_perception-aware_2020} or refine~\cite{zhou_raptor_2021} perception-aware paths, but these approaches can be computationally intensive and the solution convergence is not guaranteed. Other approaches first plan a path and then optimize the heading angle to enhance safety based on sensing objectives~\cite{murali_perception-aware_2019, bena_safety-aware_2023, chen_apace_2024}, although they do not always guarantee safety. Another method computes local safe regions during the exploration of candidate paths to account for local sensing information and ensure safety~\cite{oriolo_srt_2004}, but it requires storing all combinations of such regions, making the planning module overly computationally expensive. \revsecond{Additionally, belief-space planning approaches~\cite{van_den_berg_lqg-mp_2011, yang_anytime_2016} explicitly model state-estimation uncertainty within sampling-based frameworks. However, these approaches primarily focus on sensing uncertainty rather than limited field-of-view, which is the focus of our work.}

In this paper, we propose the Visibility-Aware RRT* algorithm that explicitly reasons about the robot's limited sensing capabilities and combines the benefits of sampling-based planning with CBFs to enable safe and efficient navigation in partially unknown environments. Our approach incorporates two types of CBFs: (i) a collision avoidance CBF to ensure that the planned path is collision-free with respect to known obstacles, and (ii) a novel visibility CBF, which we introduce in this work, that guarantees the local tracking controller following the resulting path will always keep the robot within locally collision-free regions. This ensures that the robot can detect unknown obstacles in a timely manner, enabling the robot to avoid them by virtue of the local controller. We demonstrate the effectiveness of our method through extensive simulations and experiments (see Fig.~\ref{fig:main}), comparing it with other planning methods. The results demonstrate that our approach effectively addresses the challenges faced by both CBF-based controller~\cite{ames_control_2019} and Gatekeeper algorithm~\cite{agrawal_gatekeeper_2024}, outperforming existing baselines across multiple scenarios.

%% file: _II.Preliminaries/a_dynamics.tex
We consider a robot that is modeled as a continuous-time, control-affine system:
\begin{equation}
    \dot{\vx} = f(\vx) + g(\vx) \vu ,
    \label{eq:dynamics}
\end{equation}
where $ \vx \in \StateSpace \subset \Rn $ is the state, and $ \vu \in \ControlSpace \subset \Rm $ is the control input, with $ \ControlSpace $ being a set of admissible controls for System~\eqref{eq:dynamics}. The functions $f: \StateSpace \to \Rn$ and $g: \ControlSpace \to \Rnm$ are assumed to be locally Lipschitz continuous.

We assume the robot has limited sensing capabilities, modeled via a field-of-view angle $\theta_{\text{fov}} < 360^\circ$ and a maximum sensing range $l_{\text{range}} < \infty$ of the onboard sensor.

%% file: _II.Preliminaries/c_cbf.tex
\begin{definition}[Control Barrier Function (CBF) \cite{ames_control_2019}]
   Let $\calC = \{\vx \in \StateSpace : h(\vx) \geq 0\}$, where $h: \StateSpace \rightarrow \RealSpace$ is a continuously differentiable function. The function $h$ is a CBF on $\calC$ for System~\eqref{eq:dynamics} if there exists a class $\calK$ function $\alpha$ such that
   \begin{equation}
       \sup_{\vu \in \ControlSpace} \left[\lieder_f h(\vx) + \lieder_g h(\vx)\vu\right] \geq -\alpha(h(\vx)),
   \label{eq:cbf-constraint}
   \end{equation}
   for all $\vx \in \StateSpace$ and $\lieder_f h(\vx)$ and $\lieder_g h(\vx)$ denote the Lie derivatives of $h$ along $f$ and $g$, respectively.
\end{definition}

\begin{theorem}\cite{ames_control_2019}
   Given a CBF $h$ with the associated set $\calC$, any Lipschitz continuous controller $\vu \in \calK_\textup{cbf}(\vx)$, with $\calK_\textup{cbf}(\vx) \coloneqq \{\vu \in \ControlSpace : \lieder_f h(\vx) + \lieder_g h(\vx)\vu + \alpha(h(\vx)) \geq 0\}$, renders $\calC$ forward invariant for System~\eqref{eq:dynamics}. 
\label{thm:cbf-forward-invariance}
\end{theorem}

%% file: _II.Preliminaries/d_hocbf.tex
The concept of CBFs has been extended to High Order CBFs (HOCBFs) \cite{xiao_control_2019}, which provide a more general formulation that can be used for constraints of high relative degree.

For a $r^\textup{th}$ order differentiable function $h: \StateSpace \rightarrow \RealSpace$, we define a series of functions $\psi_i : \StateSpace  \rightarrow \RealSpace$, $i = 0, \ldots, r$, as:
\begin{subequations}
\label{eq:hocbf-functions}
\begin{align}
   \psi_0(\vx) &\coloneqq h(\vx),  \label{eq:hocbf-functions-h}\\
   \psi_1(\vx) &\coloneqq \dot{\psi}_0(\vx) + \alpha_1(\psi_0(\vx)),  \\
   &\vdots \nonumber \\
   \psi_{r}(\vx) &\coloneqq \dot{\psi}_{r-1}(\vx) + \alpha_{r}(\psi_{r-1}(\vx)),
\end{align}
\label{eq:hocbf-functions-series}
\end{subequations}
where $\alpha_{i}: \RealSpace^{+} \rightarrow \RealSpace^{+}$, $i = 1, \ldots, r$, denote class $\calK$ functions, and $\dot{\psi}_{i}(\vx) = \lieder_f \psi_{i}(\vx) + \lieder_g \psi_{i}(\vx)\vu$. We further define a series of sets $\calC_i$ associated with these functions as:
\begin{subequations}
\label{eq:hocbf-sets} 
\begin{align}
   \calC_1 &\coloneqq \{\vx \in \StateSpace : \psi_0(\vx) \geq 0\},  \\
   \calC_2 &\coloneqq \{\vx \in \StateSpace : \psi_1(\vx) \geq 0\},  \\
   &\vdots \nonumber \\
   \calC_{r} &\coloneqq \{\vx \in \StateSpace : \psi_{r-1}(\vx) \geq 0\}.
\end{align}
\end{subequations}
\begin{definition}[HOCBF \cite{xiao_control_2019}]
   The function $h: \StateSpace \rightarrow \RealSpace$ is a HOCBF of relative degree $r$ for System~\eqref{eq:dynamics} if there exist differentiable class $\calK$ functions $\alpha_1, \alpha_2, \ldots, \alpha_{r}$ such that
   \begin{equation}
   \label{eq:hocbf-constraint}
       \psi_{r}(\vx) \geq 0, \quad    \forall \vx \in \calC_1 \cap \calC_2 \cap \ldots \cap \calC_{r}.
\end{equation}

\end{definition}

\begin{theorem}\cite{xiao_control_2019}
   Given a HOCBF $h$ with the associated sets $\calC_1, \calC_2, \ldots, \calC_{r}$, if $\vx(t_0) \in \calC_1 \cap \calC_2 \cap \ldots \cap \calC_{r}$, then any Lipschitz continuous controller $\vu(t) \in \calK_\textup{hocbf}(\vx)$, with $\calK_\textup{hocbf}(\vx) \coloneqq \{\vu \in \ControlSpace : \psi_{r}(\vx) \geq 0\}$, renders the set $\calC_1 \cap \calC_2 \cap \ldots \cap \calC_{r}$ forward invariant for System~\eqref{eq:dynamics}.
\label{thm:hocbf-forward-invariance}
\end{theorem}

%% file: _III.Problem/_intro.tex
Consider a robot modeled as the nonlinear system \eqref{eq:dynamics} navigating in an unknown environment~$\calW$. The environment~$\calW$ consists of two types of obstacles: a set $\calO_0 \subset \mathcal W$ of known obstacles at the initial time $t_0$, and a set $\calH \subset \mathcal W$ of unknown (or ``hidden") obstacles that can be sensed by the robot on-the-fly as it navigates through the environment. The true obstacle-free space~$\calS$ is then defined as $\calS = \calW \setminus (\calO_0 \cup \calH)$. A local tracking control policy $\pi$ generates control inputs that track a nominal path from the current to reference state.

The navigation problem is divided into two hierarchical steps: 1) finding a global reference path~$p^\textup{ref}$, and 2) tracking $p^\textup{ref}$ using a local controller~$\pi$.

\textbf{Global Path Planning:} Given a goal state~$\vx_{\textup{goal}} \in \StateSpace$, the objective of the global path planner is to find a global reference path~$p^{\textup{ref}}: \StateSpace \to \calP$, represented by a set of waypoints~$\calP = \{\vx_{j}^{\textup{ref}} \in \calX \}_{j=0}^{N}$, where $\vx_0^{\textup{ref}} = \vx_0$ and $\vx_N^{\textup{ref}} = \vx_{\textup{goal}}$ are the initial and final waypoints in $\calP$, respectively. For any given state~$\vx_t$, the function yields the next waypoint that the robot should navigate towards as:
\begin{equation}
    \vx_{j}^{\textup{ref}} = p^{\textup{ref}}(\vx_t),
    \label{eq:next_waypoint}
\end{equation}
where the index~$j$ increases monotonically with time, and $\vx_{j}^{\textup{ref}} \in \calP \subset \StateSpace$ indicates the immediate waypoint for navigation. The reference path should satisfy:
\begin{equation}
    p^\textup{ref}(\vx_t) \in \calS, \quad \forall t \geq t_0, \quad \forall \vx_t \in \StateSpace.
\end{equation}
However, the true collision-free set~$\calS$ is often unknown in advance. 

\begin{definition}[Collision-Free Path]
We define the known collision-free set at $t_0$ as $\calS_0 \coloneqq \calW \setminus \calO_0$. A reference path $p^\textup{ref}$ is considered \textbf{collision-free} if it satisfies the condition:
\begin{equation}
    p^\textup{ref}(\vx_t) \in \calS_0, \quad \forall t \geq t_0, \quad \forall \vx_t \in \StateSpace.
    \label{eq:first_condition}
\end{equation}
\end{definition}

Let $\vr_t = [x_t, y_t]^\top \in \RealSpace^{2}$ and $\theta_t \in \RealSpace$ denote the position and orientation of the robot with respect to a global frame at time $t$, respectively. The sensing footprint~$\calF_t$, which represents the currently sensed area at time $t$, is defined as:
\begin{align}
\calF_t \coloneqq \bigg\{ & \vp \in \calS  \mid  \, \|\vp - \vr_t\| \leq l_{\textup{range}} \, \cap \nonumber\\
&   | \, \angle(\vp - \vr_t, \theta_t) \, | \leq \frac{\theta_{\textup{fov}}}{2} \, \cap \, \textup{LoS}(\vr_t, \vp) \bigg\},
\end{align}
where $\angle(\vp - \vr_t, \theta_t)$ denotes the angle between the vector $\vp - \vr_t$ and the robot's orientation $\theta_t$, and $\textup{LoS}(\vr_t, \vp)$ indicates that the line of sight from $\vr_t$ to $\vp$ does not intersect any obstacles; i.e., $\forall s \in [0,1],\ \vr_t + s(\vp - \vr_t) \in \calS$.

Given the accumulated sensory information collected from $t_0$ to $t$, we define the known collision-free space~$\mathcal B_t$ at $t$ as:
\begin{equation}
\calB_t \coloneqq \bigcup_{\tau=t_0}^{t} \calF_\tau \subseteq \mathcal S.
\end{equation}
Consequently, the part of the environment that has not been sensed yet is $\calS \setminus \calB_t$, called the unmeasured space at $t$.

It is important to note that the collision-free path criterion~\eqref{eq:first_condition} treats all unmeasured space~$\calS \setminus \calB_{t_0}$ at $t_0$ as free space. However, while following the path $p^\textup{ref}$ using the local controller~$\pi$ after the planning cycle, unknown obstacles $\calH$ might exist within such unmeasured spaces.

Given this context, we introduce a stronger criterion that underlines the path's visibility and safety by requiring the robot to stay within the known local collision-free space~$\calB_t$, thereby mitigating the risk posed by unknown obstacles \cite{oriolo_srt_2004}.

\begin{definition}[Visibility-Aware Path]
A reference path~$p^\textup{ref}$ is called \textbf{visibility-aware} w.r.t. the dynamics~\eqref{eq:dynamics} if under the local tracking controller $\pi$, the updated state $\vx(t+\Delta t)$ satisfies:
\begin{equation}
    \vx(t+\Delta t) \in \calB_t, \quad \forall t \geq t_0, \quad \forall \vx_{t} \in \StateSpace,
\label{eq:second_condition}
\end{equation}
where $\Delta t>0$ is a small time increment.
\end{definition}

We formally define the problem that our global path planner aims to solve:

\begin{problem}[]
Given a path planning problem ($\calS_{0}, \vx_{0}, \vx_{\textup{goal}}$) for System~\eqref{eq:dynamics}, find a path $p^\textup{ref}$ such that for all $j = 0, \ldots, N-1$, the state and input trajectories denoted as~$\bm{\sigma}_{j} := (\vx(\tau), \vu(\tau))$ connecting consecutive waypoints $\vx^\textup{ref}_{j}$ and $\vx^\textup{ref}_{j+1}$ during the time interval $[t_j, t_{j+1})$ satisfies $\vx(\tau) \in \calC_{\textup{col}} \cap \calC_{\textup{vis}}$ and $\vu(\tau) \in \calK_{\textup{col}}(\vx(\tau)) \cap \calK_{\textup{vis}}(\vx(\tau))$ for all $\tau \in [t_j, t_{j+1})$. If such a path exists, return it; otherwise, report failure.
\label{prob:prob}
\end{problem} 

In the above formulation, $\calC_{\textup{col}}$ and $\calC_{\textup{vis}}$ are the sets that can be rendered forward invariant for System~\eqref{eq:dynamics} under the two CBFs, which encode collision-avoidance~\eqref{eq:first_condition} and visibility-awareness~\eqref{eq:second_condition}, respectively. The sets $\calK_{\textup{col}}(\vx)$ and $\calK_{\textup{vis}}(\vx)$ represent the admissible control inputs for System~\eqref{eq:dynamics} that ensure these CBF constraints hold at each state $\vx$. In the following sections, we detail how the global path planner imposes these two constraints~\eqref{eq:first_condition}, \eqref{eq:second_condition} in its search (see Sec.~\ref{sec:planner}), and illustrate how these constraints ensure that any newly detected, previously unknown obstacles remain avoidable by the local tracking controller (see Sec.~\ref{sec:controller}).

\textbf{Local Tracking Control:} A local control policy~$\pi$ is used to stabilize the systems around the subsequent waypoint~$\vx_{j}^{\textup{ref}}$~\eqref{eq:next_waypoint} and ensure that the robot remains within the true collision-free set~$\calS$. The local controller is assumed to be capable of tracking the path between subsequent waypoints with a tracking error bounded by a maximum displacement~$\epsilon$. As $p^{\textup{ref}}$ continuously updates the waypoint, it drives the robot towards the goal state~$\vx_{\textup{goal}}$. There exist many methodologies to design such controllers~\cite{ames_control_2019, agrawal_gatekeeper_2024}. Since the primary focus of this paper is on designing a visibility-aware path planner, we will discuss two case studies of such controllers in Sec.~\ref{sec:controller} and Sec.~\ref{sec:experiments} to highlight the advantages of our proposed planner when utilized alongside various local controllers.

%% file: _IV.Methodology/_intro.tex
In this section, we introduce the Visibility-Aware RRT* algorithm, which aims to solve the motion planning problem under limited sensing capabilities while satisfying the visibility condition given in \eqref{eq:second_condition}. This algorithm builds upon the LQR-RRT* framework~\cite{perez_lqr-rrt_2012} by incorporating two CBFs that guide the search process to ensure satisfaction of both the collision-free~\eqref{eq:first_condition} and visibility-aware~\eqref{eq:second_condition} properties. 

Following a common practice in motion planning problems, we use the (kinematic) unicycle model in the global path planner as a simple dynamical model, and the dynamic unicycle model, which takes translational acceleration as input instead of velocity, in the tracking controller to effectively capture more realistic dynamics and constraints of the robot.

%% file: _IV.Methodology/a_visibility_rrt.tex
The LQR-RRT*~\cite{perez_lqr-rrt_2012} is a sampling-based motion planning algorithm that incorporates the Linear Quadratic Regulator~(LQR) controller as a steering function for nonlinear systems with non-holonomic constraints. Recently, Yang \textit{et al.}~\cite{yang_lqr-cbf-rrt_2023} proposed an extension to the LQR-RRT* that incorporates CBF constraints during the planning process. Instead of formulating and solving Quadratic Programs (QPs) iteratively, as done in existing works \cite{yang_sampling-based_2019}, it checks whether the CBF constraints are satisfied and uses them as the termination condition for the steering process.

In our work, we modify the LQR-RRT* baseline to include the robot's orientation angle $\theta$ at each node, in addition to the position coordinates $(x, y)$. This modification allows it to consider the robot's rotational motion and field of view during the planning process, essential for accommodating the robot's sensing limitations.

The overall algorithm, as shown in Alg.~\ref{alg:Visibility-RRT*}, iteratively expands a tree~$(\calV, \calE)$ in the configuration space~$\calW$ until a maximum number of iterations (denoted as $\texttt{maxIter}$) is reached. The components of the algorithm are as follows:

\begin{algorithm}[t]\scriptsize
\textbf{Initialization:}  $\calV \leftarrow$ $\{\vx_0\}$; $\calE \leftarrow \emptyset$;

\For{$k=1, \ldots, \texttt{\textup{maxIter}}$}
{
    $\vx_{\textup{rand}} \leftarrow \texttt{Sample}()$;\\
    $\vx_{\textup{nearest}} \leftarrow \texttt{Nearest}(\calV, \vx_{\textup{rand}})$;\\
    $\vx_{\textup{rand}} \leftarrow \texttt{SetAngle}(\vx_{\textup{rand}}, \vx_{\textup{nearest}})$;\\
    $\vx_{\textup{new}}, \_ \leftarrow \texttt{LQR-CBF-Steer}(\vx_{\textup{nearest}}, \vx_{\textup{rand}})$;\\
    $\calX_{\textup{near}} \leftarrow \texttt{NearbyNode}(\calV, \vx_{\textup{new}})$;\\
    $\vx_{\textup{min}} \leftarrow \texttt{ChooseParent}(\calX_{\textup{near}}, \vx_{\textup{new}})$;\\
    \If{$\vx_{\textup{min}} \neq \texttt{\textup{None}}$}
    {
        $\calV \leftarrow \calV \cup \{\vx_{\textup{new}}\}$;\\
        $\calE \leftarrow \calE \cup \{(\vx_{\textup{min}}, \vx_{\textup{new}})\}$;\\
        $\texttt{Rewire}(\calX_{\textup{near}}, \vx_{\textup{new}})$;\\
    }
}
$\calP \leftarrow \texttt{ExtractPath}(\calV, \vx_{\textup{0}}, \vx_{\textup{goal}})$;\\
\Return $\calP$;
\caption{\texttt{Visibility-Aware RRT*}}
\label{alg:Visibility-RRT*}
\end{algorithm}

\begin{algorithm}[t]\scriptsize
\texttt{minCost} $\leftarrow \infty$; $\vx_{\textup{min}} \leftarrow \texttt{\textup{None}}$;\\
\For{$\vx_{\textup{near}} \in \mathcal{X}_{\textup{near}}$}
{
    $\_, \bm{\sigma} \leftarrow \texttt{LQR-CBF-Steer}(\vx_{\textup{near}}, \vx_{\textup{new}})$;\\
    \If{$\bm{\sigma} \neq \texttt{\textup{None}}$}
    {
        \If{$\vx_{\textup{near}}.\texttt{\textup{cost}} + \texttt{\textup{Cost}}(\bm{\sigma}) < \texttt{\textup{minCost}}$}
        {
            \texttt{minCost} $\leftarrow \vx_{\textup{near}}.\texttt{cost} + \texttt{Cost}(\bm{\sigma})$;\\
            $\vx_{\textup{min}} \leftarrow \vx_{\textup{near}}$;\\
        }
    }
}
$\vx_{\text{new}}.\texttt{parent} \leftarrow \vx_{\text{min}}$;\\
$\vx_{\text{new}}.\texttt{cost} \leftarrow \texttt{minCost}$;\\
\Return $\vx_{\textup{min}}$;
\caption{\texttt{ChooseParent}($\mathcal{X}_{\textup{near}}$, $\vx_{\textup{new}}$)}
\label{alg:ChooseParent}
\end{algorithm}

\begin{itemize}
\item \texttt{Sample}: Uniformly samples a random state~$\vx_\text{rand} = [x_\text{rand}, y_\text{rand}, \theta_\text{rand}]^\top$ within the configuration space to expand the tree towards.

\item \texttt{Nearest}: Finds the nearest node~$\vx_\text{nearest}$ in the current vertices in the tree~$\calV$ to the sampled state~$\vx_\text{rand}$.

\item \texttt{SetAngle}: Assigns the orientation angle~$\theta_\text{rand}$ to the sampled state~$\vx_{\text{rand}}$ based on the direction from the nearest node $\vx_{\text{nearest}}$ to $\vx_{\text{rand}}$: $\theta_{\text{rand}} = \arctan\left(\frac{y_{\text{rand}} - y_{\text{nearest}}}{x_{\text{rand}} - x_{\text{nearest}}}\right)$.

\item \texttt{LQR-CBF-Steer}: Generates a new node $\vx_\text{new}$ by steering from $\vx_\text{nearest}$ towards $\vx_\text{rand}$. The details are described below.

\item \texttt{NearbyNode}: Identifies a set of nearby nodes~$\calX_{\text{near}}$ within a pre-defined euclidean distance from $\vx_{\text{new}}$.

\item \texttt{ChooseParent} (Alg.~\ref{alg:ChooseParent}): Selects the best parent node~$\vx_\text{min}$ for the new node~$\vx_\text{new}$ from the set of nearby nodes~$\mathcal{X}_\text{near}$ based on the cost and the feasibility of the connecting trajectory generated by the steering function. Given a trajectory $\bm{\sigma} = (\vx(t), \vu(t))$, where $t \in [0, T]$ and $T$ is the duration of the trajectory, the evaluated quadratic cost is defined as:
\begin{equation}
 \texttt{cost}(\bm{\sigma}) = \int_{0}^{T} \vx(t)^\top Q \vx(t) + \vu(t)^\top R \vu(t) \, \mathrm{d}t, 
    \label{eq:trajectory_cost}
\end{equation}
where $Q \succeq 0$ and $R \succ 0$ are weight matrices. 

\item \texttt{Rewire} (Alg.~\ref{alg:Rewire}): Performs the rewiring process by considering the new node~$\vx_\text{new}$ as a potential parent for the nearby nodes in $\calX_{\text{near}}$. If a trajectory from $\vx_{\text{new}}$ to any nearby node~$\vx_{\text{near}} \in \calX_{\text{near}}$ is found with a smaller cost~\eqref{eq:trajectory_cost}, the tree is optimized by rewiring $\vx_{\text{near}}$ to $\vx_{\text{new}}$.

\item \texttt{ExtractPath}: Searches for a path from $\vx_\text{goal}$ to $\vx_\text{start}$ within the tree~$(\calV, \calE)$. If a feasible path is found, it returns the solution path as a set of waypoints $\calP$.
\end{itemize}

\begin{algorithm}[t]\scriptsize
\For{$\vx_{\textup{near}} \in \calX_{\textup{near}}$}
{
    $\_, \bm{\sigma} \leftarrow \texttt{LQR-CBF-Steer}(\vx_{\textup{new}}, \vx_{\textup{near}})$;\\
    \If{$\bm{\sigma} \neq \texttt{\textup{None}}$}
    {
        \If{$\vx_{\textup{new}}.\texttt{\textup{cost}} + \texttt{\textup{Cost}}(\bm{\sigma}) < \vx_{\textup{near}}.\texttt{\textup{cost}}$}
        {
            $\vx_{\textup{near}}.\texttt{parent} \leftarrow \vx_{\textup{new}}$;\\
            $\vx_{\textup{new}}.\texttt{cost} \leftarrow \texttt{\textup{Cost}}(\bm{\sigma})$;\\
            
        }
    }
}
\caption{\texttt{Rewire}($\calX_{\textup{near}}$, $\vx_{\textup{new}}$)}
\label{alg:Rewire}
\end{algorithm}

\begin{algorithm}[t]\scriptsize
$\calX' \leftarrow [\vx_{\textup{start}}]$; $\calU' \leftarrow []$; $\vx'_0 \leftarrow \vx_{\textup{start}} $;\\
$K_\textup{lqr} \leftarrow \texttt{LQRSolver}(\vx_{\textup{next}})$\\
\For{$t = 1, \ldots, T$}
{
    $\vx'_{t}, \vu'_{t} \leftarrow \texttt{Integrator}(\vx'_{t-1}, K_\textup{lqr})$; \\
    \If{ $\forall \, \texttt{\textup{CBFConstraint}}(\vx'_{t}, \vu'_{t})$ are satisfied}
    {
        $\calX'.\texttt{append}(\vx'_{t})$; $\calU'.\texttt{append}(\vu'_{t})$;\\
    }
    \Else
    {
        $\textbf{break}$;
    }
}
$\vx_{\textup{new}} \leftarrow \vx'_{t}$;\\
\Return $\vx_{\textup{new}}, \bm{\sigma} = (\calX', \calU')$;\\
\caption{\texttt{LQR-CBF-Steer}($\vx_{\textup{start}}, \vx_{\textup{next}}$)}
\label{alg:lqr_cbf_steer}
\end{algorithm}

The \texttt{LQR-CBF-Steer} is the key function that allows the algorithm to use LQR and CBF~\eqref{eq:cbf-constraint} as an extension heuristic between two nodes $\vx_{\text{start}}$ and $\vx_{\text{next}}$ (see Alg.~\ref{alg:lqr_cbf_steer}). It first linearizes the dynamics around $\vx_\text{next}$ and computes the optimal gain~$K_\text{lqr}$ by solving the algebraic Riccatti equation. The function then employs the LQR controller~$\pi_{\text{lqr}}$ to generate a sequence of intermediate states $\calX' = \{\vx_{\text{start}}, \vx'_1, \ldots, \vx'_T\}$ and control inputs $\calU' = \{\vu'_1, \ldots, \vu'_T\}$, steering the state from $\vx_{\text{start}}$ towards $\vx_{\text{next}}$ based on the linearized dynamics. Throughout this process, it verifies that all CBF constraints are satisfied for each state-action pair $(\vx'_t, \vu'_t)$ at each discretized time step $t$. If any constraint is violated, the steering extension is immediately terminated \cite{yang_lqr-cbf-rrt_2023}.

One of the main advantages of this method is its efficiency in handling multiple CBF constraints. By avoiding the need to solve QPs iteratively, the algorithm significantly reduces computational overhead, and also circumvents the issue of recursive feasibility that arises from the parameters of CBF and QP \cite{parwana_recursive_2022}. In the following sections, we introduce two types of CBFs to ensure that the resulting path is both \textit{\textbf{collision-free}}~\eqref{eq:first_condition} and \textit{\textbf{visibility-aware}}~\eqref{eq:second_condition}.

%% file: _IV.Methodology/b_cbf.tex
The unicycle model is as follows:
\begin{align}
\dot{\vx} &= \begin{bmatrix}
\cos\theta & 0\\
\sin\theta & 0\\
0 & 1
\end{bmatrix} 
\begin{bmatrix}
v \\
\omega
\end{bmatrix} = g(\vx) \vu,
\label{eq:unicycle-dynamics}
\end{align}
where $\vx = [x, y, \theta]^\top \in \StateSpace$ is the state vector and the control input $\vu = [v, \omega]^\top \in \ControlSpace$ represents the translational velocity~$v$ and angular velocity~$\omega$. 

Hereafter, we abuse the notation and use $\vx_t$ and $\vu_t$ to refer to $\vx'_t$ and $\vu'_t$ in the context of the \texttt{LQR-CBF-Steer} function. Let one of the obstacles with radius~$l_{\text{obs}} \in \RealSpace^{+}$ be located at $(x_{\text{obs}}, y_{\text{obs}}) \in \RealSpace^{2}$. The constraint function~\eqref{eq:hocbf-functions-h} for collision avoidance, considering the robot radius~$l_\text{robot}$ and the maximum tracking error~$\epsilon$ mentioned in Sec.~\ref{sec:problem}, is:
\begin{equation}
h_{\textup{col}}(\vx_{t}) = (x_{t} - x_{\text{obs}})^2 + (y_{t} - y_{\text{obs}})^2 - (l_{\text{obs}} + l_{\text{robot}} + \epsilon)^2.
\label{eq:collision-cbf}
\end{equation}
Given the dynamics~\eqref{eq:unicycle-dynamics}, the function~$h_{\textup{col}}(\vx_{t})$ in \eqref{eq:collision-cbf} has relative degree $r=1$ for the control input $v$ and $r=2$ for the control input $\omega$. To simplify the construction of the CBF constraint, we fix the translational velocity $v$ to a constant value and only control the angular velocity $\omega$~\cite{yang_sampling-based_2019}.

Then, the HOCBF constraint~\eqref{eq:hocbf-constraint} with $r=2$ is~\cite{yang_sampling-based_2019}:
\begin{align}
\psi_{\textup{col}}(\vx_{t}) &= 2(x_{t} - x_{\text{obs}})v^2\cos^2{\theta_{t}}+2(y_{t} - y_{\text{obs}})v^2\sin^2{\theta_{t}}\nonumber\\
&+[2(y_{t} - y_{\text{obs}})v\cos{\theta_{t}}-2(x_{t} - x_{\text{obs}})v\sin{\theta_{t}}]\omega_{t} \nonumber\\
&+\gamma_1 \lieder_f h_{\textup{col}}(\vx_{t})+\gamma_2 h_{\textup{col}}(\vx_{t}) \geq 0,
\label{eq:collision-cbf-constraint}
\end{align}
where $\gamma_1$ and $\gamma_2$ are positive constants that are designed to satisfy the HOCBF condition~\eqref{eq:hocbf-constraint}, and
\begin{equation}
\lieder_f h_{\textup{col}}(\vx_{t}) = 2v(x_{t} - x_{\text{obs}})\cos{\theta_{t}} + 2v(y_{t} - y_{\text{obs}})\sin{\theta_{t}}.
\end{equation}
Constraint~\eqref{eq:collision-cbf-constraint} is then incorporated into the \texttt{LQR-CBF-Steer} function, ensuring that every node added to the tree is collision-free, as illustrated in Fig.~\ref{fig:cbf}a.

\begin{figure}[t]
    \centering
    \includegraphics[width=0.85\linewidth]{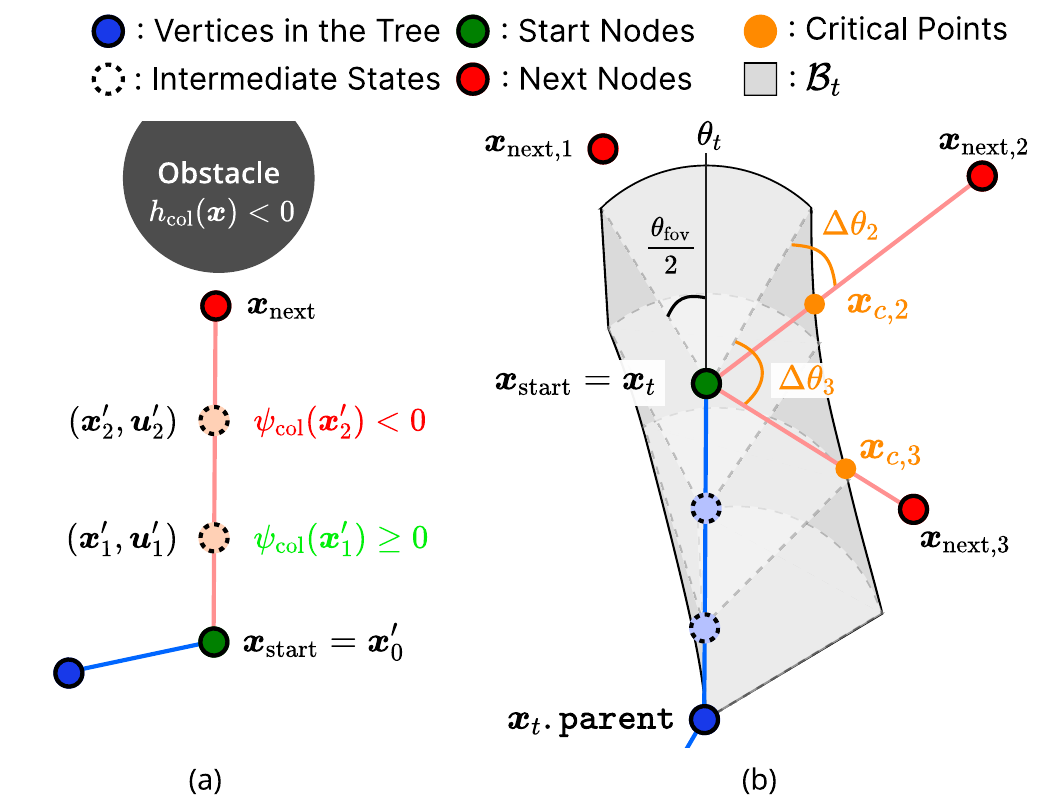}
    \caption{Illustration of the CBF constraints in the \texttt{LQR-CBF-Steer} function. The \texttt{Integrator} generates a sequence of intermediate states using the LQR gain~$K_\text{lqr}$. At each intermediate state, the \texttt{CBFConstraint}s are checked to ensure that the state satisfies the given constraints. If any constraint is violated, the steering process is terminated, and the last satisfied state is returned as the new node. (a) An example demonstrating the collision avoidance HOCBF constraint~\eqref{eq:collision-cbf-constraint}. (b) An example demonstrating the visibility CBF constraint~\eqref{eq:visibility-cbf-constraint}. $\vx_{\text{next},1}$ is already within the FOV at $\vx_t$, satisfying the visibility constraint. For $\vx_{\text{next},2}$, although it is outside the current FOV, it satisfies the visibility constraint. $\vx_{\text{next},3}$ might satisfy the visibility constraint $h_{\text{vis}}(\vx_t) \geq 0$~\eqref{eq:visibility-cbf}, but it violates the visibility CBF constraint $\psi_{\textup{vis}}(\vx_t) \geq 0$~\eqref{eq:visibility-cbf-constraint}, causing the steering to terminate.}
    \label{fig:cbf}
\end{figure}

%% file: _IV.Methodology/c_visibility.tex
The collision-free path planning discussed in Section~\ref{subsec:cbf-collision-avoidance} assumes the environment is fully known, i.e., $\calS_0 = \calS$. However, partially-known environments and limited perception, such as restricted FOV and sensing range, can pose significant challenges in ensuring safe navigation. To mitigate the risks posed by these limitations, we introduce a visibility constraint formulated as a CBF to ensure that the generated path allows for any newly-detected obstacles to be avoided by a local tracking controller.
\begin{definition}[Critical Point] A critical point, denoted as $\vx_c = [x_c, y_c, \theta_c]^\top$, is the intersecting point between the frontier~$\partial \calB_t$, which represents the boundary of the explored collision-free space~$\calB_t$ up to the current time~$t$, and the path from the start node $\vx_{t}$ to the next node $\vx_\textup{next}$. The angle $\theta_c$ represents the direction from $\vx_{t}$ to $\vx_c$.
\label{def:critical-point}
\end{definition}
The critical point~$\vx_c$ represents the last sensed point along the linearized path (as illustrated in Fig.~\ref{fig:cbf}b), beyond which the robot has no updated information about potential unknown obstacles. The goal is to enforce the robot to have seen the critical point within its FOV before reaching it. To achieve this, we formulate a visibility constraint based on the time required for the robot to rotate towards and observe the critical point, denoted as $t_\text{rot}(\theta_{t})$, and the time needed to reach it, denoted as $t_\text{reach}(x_{t}, y_{t})$. The visibility constraint requires that $t_\text{reach}(x_{t}, y_{t}) \geq t_\text{rot}(\theta_{t})$, guaranteeing that the robot will have sufficient time to detect any potential hidden obstacles at the critical point and allow the local tracking controller~$\pi$ to take appropriate actions if necessary.

\subsubsection{Time-to-Reach} Given the robot's current position~$\vr_{t} = [x_{t}, y_{t}]^\top$ and the critical point~$\vx_c$, let $\Delta \vr = \sqrt{(x_{t} - x_c)^2 + (y_{t} - y_c)^2}$ denote the Euclidean distance between them. To account conservatively for the robot radius~$l_\text{robot}$ and the maximum tracking error~$\epsilon$ of the local controller~$\pi$, we define $\Delta d = \Delta \vr - l_\text{robot} - \epsilon$. The time-to-reach is:
\begin{equation}
t_\text{reach}(x_{t}, y_{t}) = \frac{\Delta d}{v},
\label{eq:t_reach}
\end{equation}
where $v$ is the constant velocity assumed during the planning cycle as mentioned in Sec.~\ref{subsec:cbf-collision-avoidance}.

\subsubsection{Time-to-Rotate}
Given the current robot orientation~$\theta_{t}$, the amount of angle that the robot should rotate to observe the critical point using the onboard sensor is:
\begin{equation}
\Delta \theta = |\theta_{t} - \theta_c| - \frac{\theta_{\text{fov}}}{2}.
\end{equation}
Using the cosine difference, we can rewrite $\Delta \theta$ as:
\begin{equation}
\Delta \theta = \arccos{(\cos{\theta_{t}}\cos{\theta_c} + \sin{\theta_{t}}\sin{\theta_c})} - \frac{\theta_{\text{fov}}}{2}.
\end{equation}
Then, time-to-rotate~$t_\text{rot}(\theta_t)$ can be obtained as:
\begin{equation}
t_\text{rot}(\theta_t) = \frac{\Delta \theta}{\bar{\omega}_{t}},
\end{equation}
where $\bar{\omega}_{t}$ is the average angular velocity computed by simulating the LQR controller~$\pi_{\text{lqr}}$ for a rotation of $\Delta \theta$ starting from time $t$, i.e., $\bar{\omega}_{t} = \frac{1}{T_\omega} \sum_{\tau=t}^{t + T_\omega} \omega_{\tau}^{\text{lqr}}$. 
$\omega_{\tau}^{\text{lqr}}$ denotes the angular velocity generated by the LQR controller~$\pi_{\text{lqr}}$ at time step $\tau$, and $T_\omega$ denotes the total number of time steps required for the rotation.

\subsubsection{Visibility CBF}
Based on the visibility constraint $t_\text{reach}(x_{t}, y_{t}) \geq t_\text{rot}(\theta_{t})$, we define the visibility CBF candidate $h_{\textup{vis}}:\mathbb R^3\rightarrow \mathbb R$ as:
\begin{equation}
h_{\textup{vis}}(\vx_{t}) = t_\text{reach}(x_{t}, y_{t}) - t_\text{rot}(\theta_{t}).
\label{eq:visibility-cbf}
\end{equation}
Then, the CBF constraint with $r=1$~\eqref{eq:cbf-constraint} is defined as:
\begin{equation}
\psi_{\textup{vis}}(\vx_{t}) = \dot{h}_{\textup{vis}}(\vx_{t}) + \gamma_3 \, h_{\textup{vis}}(\vx_{t}) \geq 0,
\label{eq:visibility-cbf-constraint}
\end{equation}
where $\gamma_3$ is positive constant that is designed to satisfy the CBF condition. The gradient~$\dot{h}_{\textup{vis}}(\vx_{t})$ can be derived as:
\begin{align}
\dot{h}_{\textup{vis}}(\vx_{t}) &= L_f h_{\textup{vis}}(\vx_{t}) + L_g h_{\textup{vis}}(\vx_{t}) \vu_{t} \\
&= \frac{x_{t}-x_c}{\Delta \vr} \cos \theta_{t} + \frac{y_{t}-y_c}{\Delta \vr} \sin \theta_{t} \nonumber \\
&\quad - \frac{1}{\bar{\omega}_{t}} \left(\frac{\sin \theta_{t} \cos \theta_c - \cos \theta_{t} \sin \theta_c}{\sqrt{1-z_{t}^2}} \right) \omega_{t},
\end{align}
where $z_{t} = \cos{\theta_{t}}\cos{\theta_c} + \sin{\theta_{t}}\sin{\theta_c}$.

\begin{proof}
Let us verify that condition \eqref{eq:cbf-constraint} holds for all $\vx \in \StateSpace$. Note that $L_f h_{\textup{vis}}(\vx_{t})$ is bounded in $[-1,1]$, and $L_g h_{\textup{vis}}(\vx_{t})$ takes the value $\pm 1/\bar{\omega}_{t}$ depending on $\theta_{t}$ and $\theta_{c}$. Plus, $\bar{\omega}_{t}$ is bounded because it is computed from the average LQR control input, which converges to zero over time, and the largest control input is bounded in the control space~$\calX$. Therefore, a sufficiently large control bound on the angular velocity~$\omega_{\max}$ satisfying $\omega_{\max} \geq \bar{\omega}_{t}$ guarantees that there exists a control input $\omega_{t}$ such that the condition~\eqref{eq:cbf-constraint} is satisfied, even at the boundary $h_{\textup{vis}}(\vx_{t})=0$. Hence, $h_{\textup{vis}}$ is a valid CBF.
\end{proof}

The visibility CBF constraint $\psi_{\textup{vis}}(\vx_{t}) \geq 0$ is incorporated into the \texttt{LQR-CBF-Steer} function, as illustrated in Fig.~\ref{fig:cbf}b, alongside the collision avoidance CBF constraint~\eqref{eq:collision-cbf-constraint}. By enforcing both constraints during the steering process, the algorithm generates paths that satisfy the feasibility of both safety conditions.

One issue in this formulation is that computing the location of the critical point~$\vx_c$ whenever the \texttt{LQR-CBF-Steer} function is called, either during tree expansion or rewiring, is computationally intensive since it requires calculating the entire local collision-free space~$\calB_t$~\cite{oriolo_srt_2004}. To mitigate this issue in the implementation, we assume that $\calB_t$ forms a linear tube along the LQR trajectory between $\vx_t$ and the parent node of $\vx_t$. This allows us to calculate the critical point~$\vx_c$ geometrically without the need to track and store the entire $\calB$ at each iteration, while still providing a reasonable approximation of the critical point.

%% file: _V.Local_Tracking_Controller/a_controller.tex
After the global planning cycle described in Sec.~\ref{sec:planner}, the controller~$\pi$ tracks the waypoints~$\vx_{j}^{\text{ref}} \in \calP$ enabling the robot to navigate towards the goal. During navigation, the robot utilizes its onboard sensor to create a map of the environment and detect obstacles on-the-fly. The environment may contain unknown obstacles that were not considered during the global planning phase. The controller must ensure that the robot remains within the true collision-free set~$\calS$ during navigation, avoiding collisions with both known and unknown obstacles.

Two main strategies exist for designing safety-critical tracking controllers in such cases. The first strategy treats unmeasured space as free space and focuses on avoiding previously known and newly detected obstacles. A common approach to this problem is the use of CBFs to formally guarantee collision avoidance. CBF-based quadratic programs (CBF-QPs) have emerged as a popular and computationally efficient method for enforcing safety-critical control objectives~\cite{ames_control_2019}. However, a notable challenge emerges under limited sensing capabilities. Specifically, if an obstacle is detected in close proximity to the robot, particularly during rotational maneuvers, the CBF-QP may become infeasible as there might not be a feasible solution to avoid the obstacle due to the proximity and the robot's dynamics.

\begin{figure*}[t]
\centering
    \subfloat[Env. 1 - Baseline 3]{\includegraphics[width=0.22\textwidth]{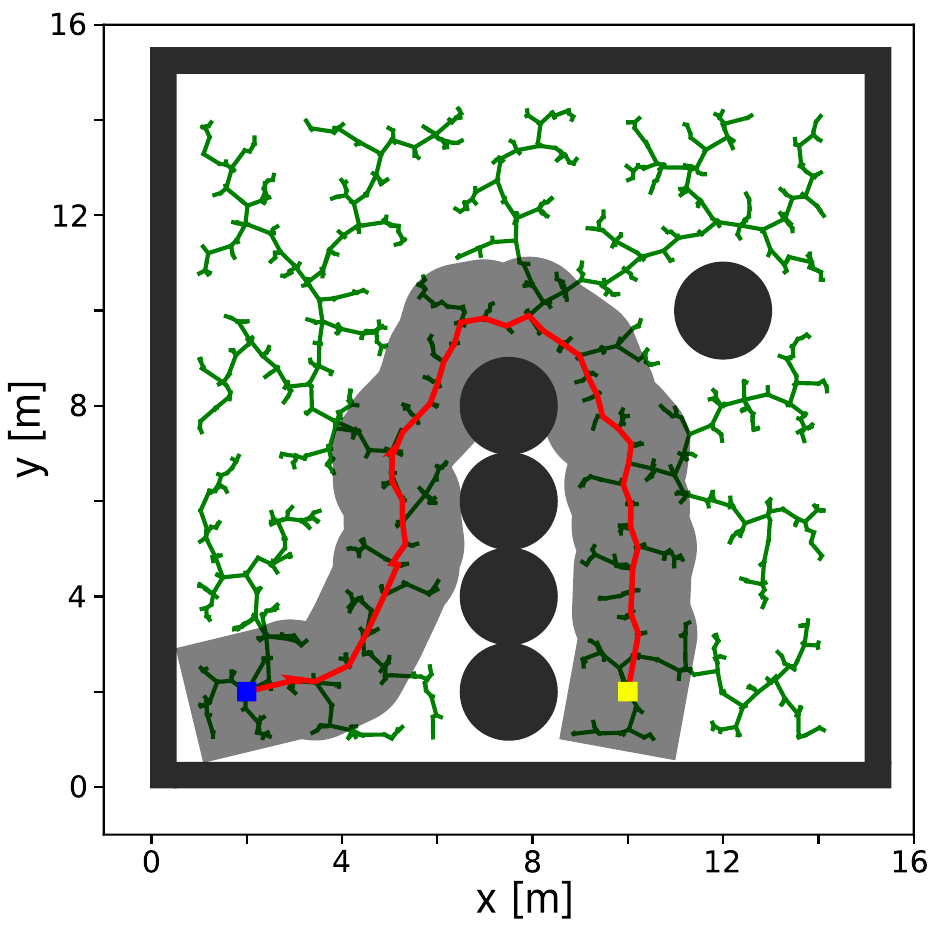}\label{fig:env1_baseline}}
    \subfloat[Env. 1 - Ours]{\includegraphics[width=0.22\textwidth]{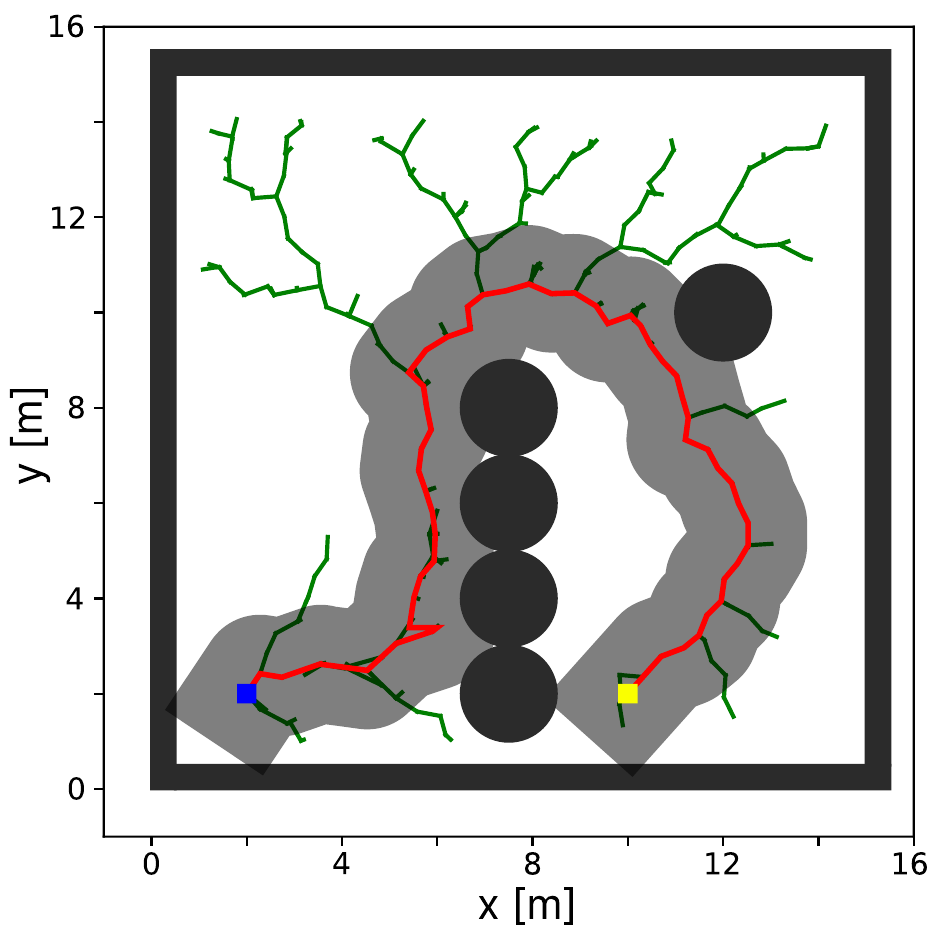}\label{fig:env1_visibility}} 
    \subfloat[Env. 2 - Baseline 3]{\includegraphics[width=0.24\textwidth]{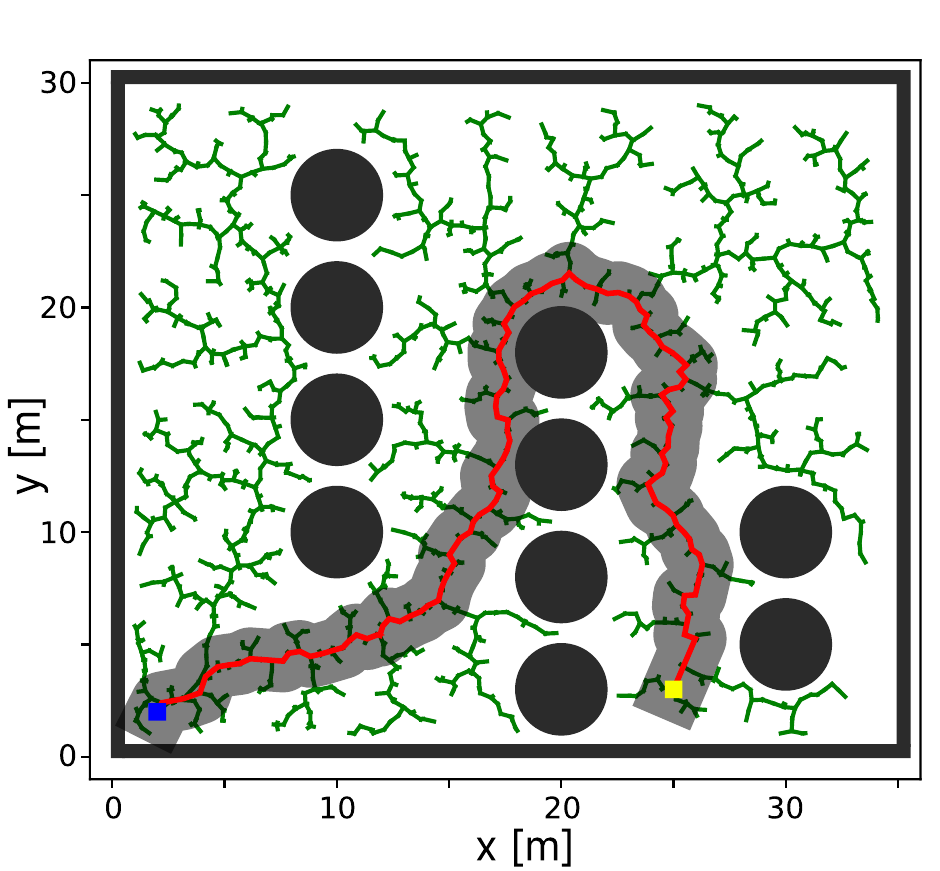}\label{fig:env2_baseline}}
    \subfloat[Env. 2 - Ours]{\includegraphics[width=0.24\textwidth]{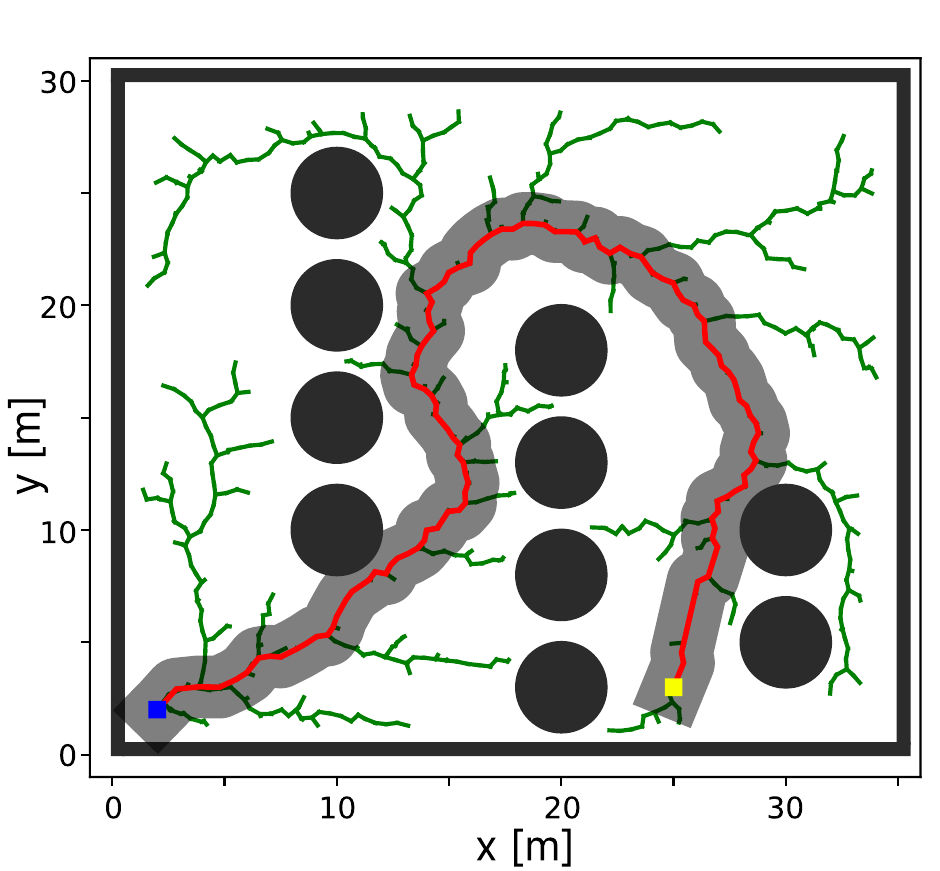}\label{fig:env2_visibility}}
    \caption{Visualization of the global planning results generated by Baseline 3 and the proposed method in two environments. $\texttt{maxIter}$ is set to 2000 and 3000 for Env. 1 and Env. 2, respectively. The blue and yellow squares represent the start and goal position. The black circles represent the known obstacles. The green lines depict the edges of the tree~$\calE$ appended during the planning process. The red line depict the final reference path. The shaded areas in gray represent the local collision-free set~$\calB_t$ that the robot will sense while following the reference path.
    }
\label{fig:planner}
\vspace{-5pt}
\end{figure*}

\begin{figure}[t]
\centering
    \subfloat[CBF-QP - Baseline 3]{\includegraphics[width=0.46\linewidth]{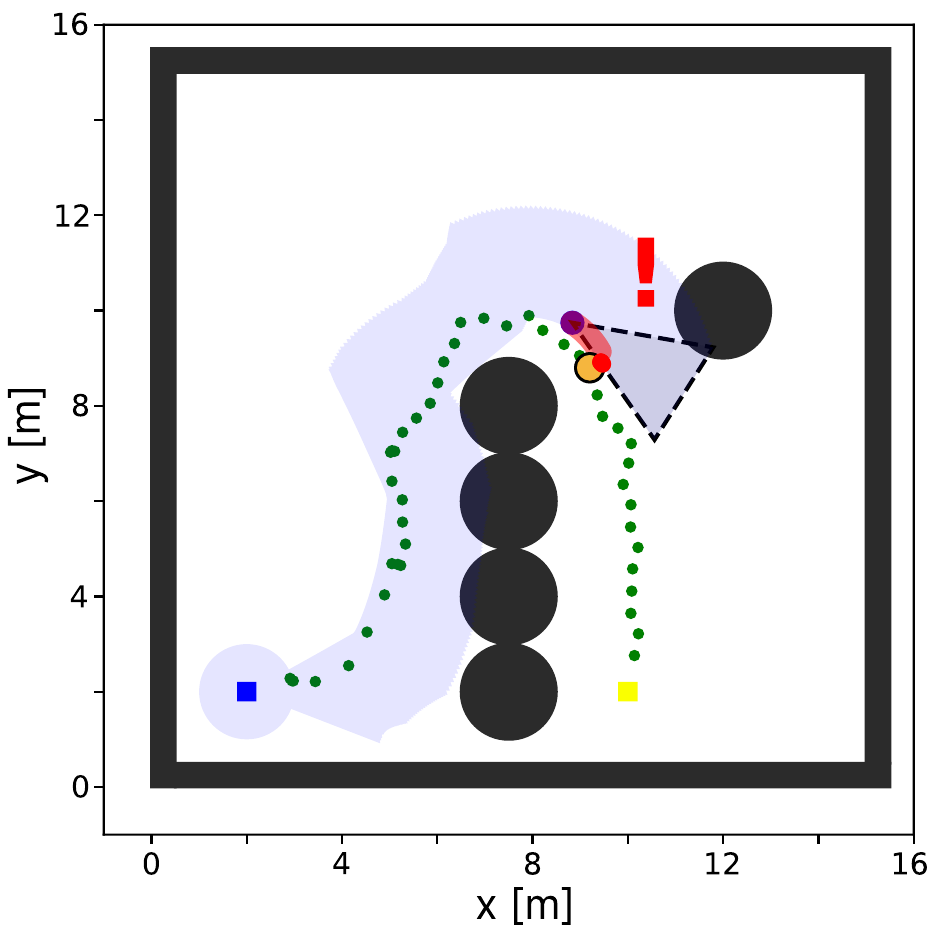}\label{fig:cbf_qp_fail}}
    \subfloat[CBF-QP - Ours]{\includegraphics[width=0.46\linewidth]{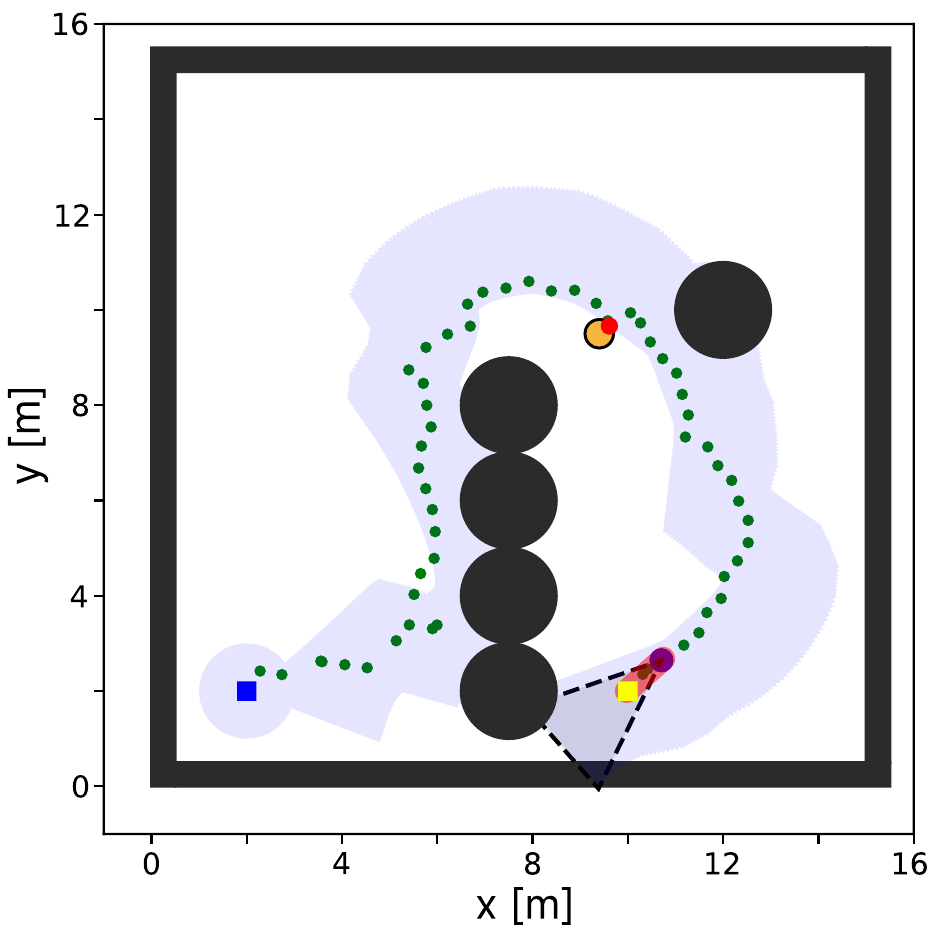}\label{fig:cbf_qp_success}} 
    \caption{
    Visualization of the CBF-QP experiments for Env. 1 with a $45^\circ$ FOV. The reference waypoints (green dots) correspond to the paths in Fig.~\ref{fig:env1_baseline} and Fig.~\ref{fig:env1_visibility}. The blue shaded areas represent the actual local collision-free set~$\calB_t$ collected from the onboard sensor. The orange circles are the hidden obstacles~$\calH$ and the red dots indicate the detection points of these hidden obstacles. (a) The CBF-QP becomes infeasible when it detects the hidden obstacle, as the robot is already within an unsafe distance to perform collision avoidance. (b) Tracking the reference path from our method, the CBF-QP successfully avoids the hidden obstacle and reaches the goal.
}
\label{fig:cbf-qp}
\vspace{-5pt}
\end{figure}

\begin{figure}[t]
\centering
    \subfloat[Gatekeeper- Baseline 3]{\includegraphics[width=0.46\linewidth]{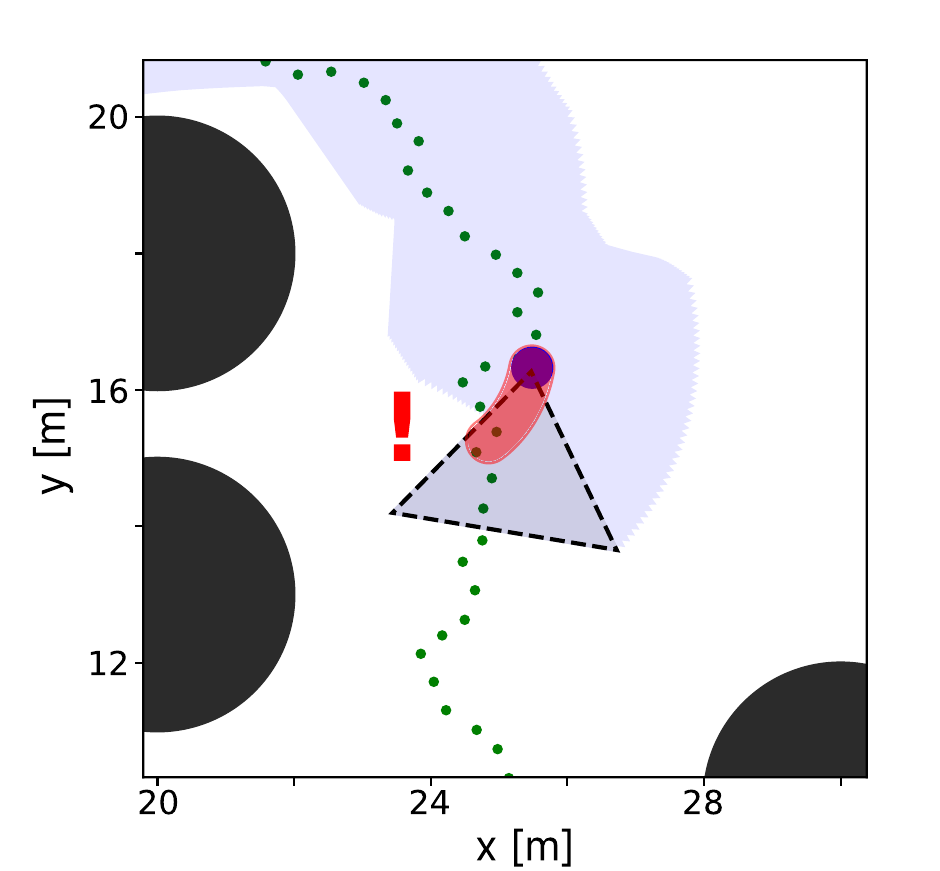}\label{fig:gatekeeper_fail}}
    \subfloat[Gatekeeper - Ours]{\includegraphics[width=0.46\linewidth]{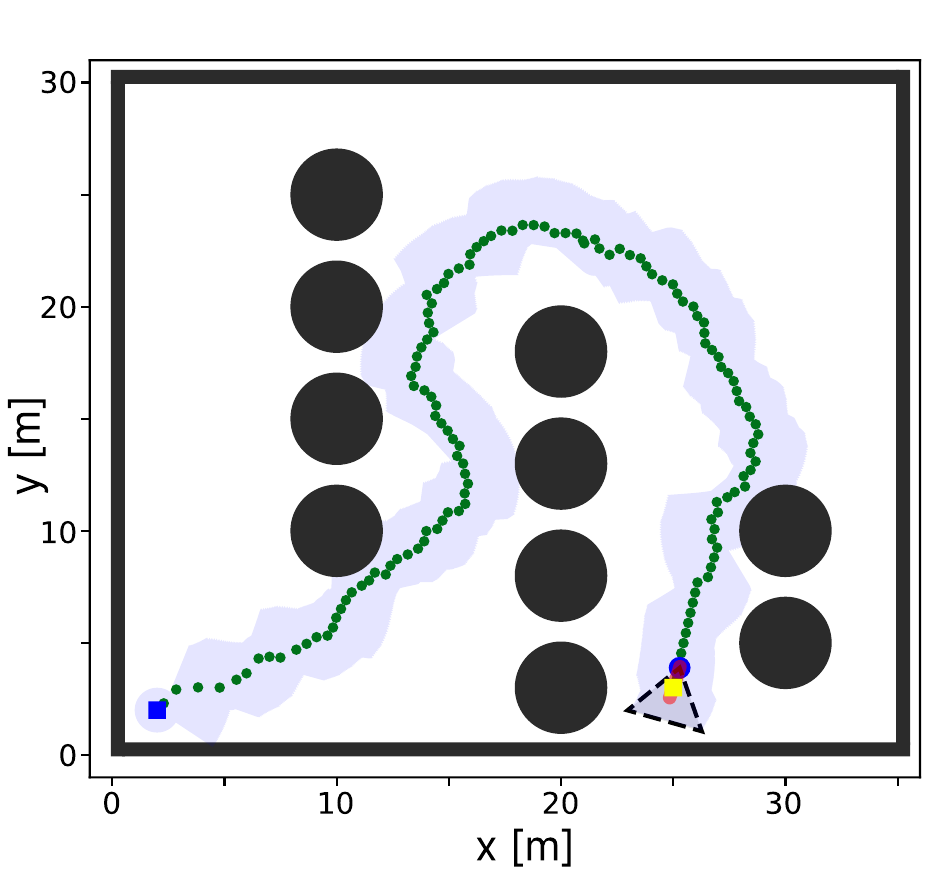}\label{fig:gatekeeper_success}} 
    \caption{Visualization of the Gatekeeper experiments for Env. 2. The reference waypoints correspond to the paths in Fig.~\ref{fig:env2_baseline} and Fig.~\ref{fig:env2_visibility}. Red shaded areas depict the subsequent nominal trajectories at the robot's position. (a) The Gatekeeper executes a stop command as the nominal trajectory is deemed unsafe, where the trajectory lies outside of the local collision-free set~$\calB_t$. Consequently, global replanning is required from the current position. (b) By tracking our reference path, the Gatekeeper successfully navigates to the goal without violating visibility constraint.}
    \label{fig:gatekeeper}
\vspace{-5pt}
\end{figure}

The second strategy treats the unmeasured space as an unsafe set and only controls the robot within the collision-free space sensed by the onboard sensor. The Gatekeeper algorithm~\cite{agrawal_gatekeeper_2024} is a recent advance in this strategy, acting as a safety filter between the planner and the low-level controller. If the nominal trajectory enters the unmeasured space, Gatekeeper executes a backup controller to ensure the robot remains within the known collision-free space. While this strategy prevents collisions with hidden obstacles if a valid backup controller exists, the main challenge arises when the nominal global planner is agnostic to the robot's local sensing capabilities. In such cases, Gatekeeper frequently executes the backup controller and requires global replanning based on the updated environmental information, which can significantly degrade the system's operational efficiency.

Assuming a safety-critical tracking controller~$\pi$ interfacing with the global planner, we can now prove the properties of the resulting path from our proposed planner:
\begin{theorem}
Let $p^\textup{ref}$ be the global reference path generated by the Visibility-Aware RRT* (Alg.~\ref{alg:Visibility-RRT*}). If there exists a local tracking controller~$\pi$ that can track the trajectory between consecutive waypoints with a maximum tracking error~$\epsilon$, then the reference path $p^\textup{ref}$, when tracked by such controller~$\pi$, is both \textit{\textbf{collision-free}} and \textit{\textbf{visibility-aware}}.
\end{theorem}
\begin{proof}
By construction, each trajectory~$\bm{\sigma}_{j}$ generated by the \texttt{LQR-CBF-Steer} function satisfies both collision avoidance~\eqref{eq:collision-cbf-constraint} and visibility~\eqref{eq:visibility-cbf-constraint} CBF constraints, considering the robot radius $l_\text{robot}$ and the maximum tracking error $\epsilon$. Since \eqref{eq:collision-cbf-constraint} is satisfied, the trajectories remain within the set $\calC_\textup{col}$. This guarantees that the controller~$\pi$ follows the trajectory~$\bm{\sigma}_{j}$ with a maximum tracking error of $\epsilon$ to remain \textit{\textbf{collision-free}}~\eqref{eq:first_condition}. Moreover, since \eqref{eq:visibility-cbf-constraint} is met, the controller~$\pi$ forces the robot to move along position trajectories $\bm{\sigma}_{j}$ from which the critical point (according to Definition~\ref{def:critical-point}) is visible before it is reached. This guarantees that the tracking controller~$\pi$ can detect and avoid hidden obstacles that, in the worst-case scenario, are placed at the boundary of the local safe set $\calB_t$. Consequently, the controller~$\pi$ can keep the robot within the known local collision-free sets~$\calB_t$ while tracking the reference path $p^\textup{ref}$, satisfying the \textit{\textbf{visibility-aware}} condition~\eqref{eq:second_condition}.
\end{proof}

%% file: _VI.Experiments/a_results.tex
In this section, we evaluate the proposed Visibility-Aware RRT* through simulation results in two distinct environments. We demonstrate that our approach can address the challenges faced by both CBF-QP and Gatekeeper methods, providing a safe and efficient solution for navigating in environments with unknown obstacles, which visibility-agnostic planners fail to achieve.

The first environment (Env. 1) spans an area of 15~m $\times$ 15~m, while the second one (Env. 2) covers a larger area of 35~m $\times$ 30~m. Both environments contain known obstacles that prevent the robot from taking a straight-line path to the goal. The onboard sensor's FOV~$\theta_{\text{fov}}$ and the sensing range~$l_{\text{range}}$ are set to $70^\circ$ and 3~m, respectively, unless otherwise specified.

\begin{table}[t]
\centering
\caption{(left) Collision rates of the CBF-QP controller tracking reference paths generated by different planners in two environments with varying FOV angles. (right) Percentage of reference paths that triggered the execution of the backup controller in the Gatekeeper algorithm, indicating the need for global replanning from scratch. Note that Baseline 1 is a deterministic planner, there is only one reference path tested.}
\label{tab:controller}
\begin{tabular}{lcccc|cc}
\toprule
Controller & \multicolumn{4}{c|}{CBF-QP} & \multicolumn{2}{c}{Gatekeeper}                       \\ 
 \cmidrule(r){2-5}\cmidrule(l){6-7}
Env. & \multicolumn{2}{c}{Env. 1} & \multicolumn{2}{c|}{Env. 2} & Env. 1 & Env. 2   \\
FOV & $45^\circ$  & $70^\circ$ & $45^\circ$ & $70^\circ$ & $70^\circ$ & $70^\circ$ \\
\midrule
Baseline 1              & 100\% & 0\% &  0\% & 0\%  & 100\%  & 100\%  \\
Baseline 2              &   5\% & 3\% & 13\% & 9\%  &  23\%  &  32\%  \\
Baseline 3              &   4\% & 0\% &  1\% & 1\%  &  18\%  &  13\%  \\
\rowcolor{gray!50} Ours &   0\% & 0\% &  0\% & 0\%  &   1\%  &   0\%  \\
\bottomrule
\end{tabular}
\vspace{-6pt}
\end{table}

We compare four planning methods in our experiments: (\romannumeral 1) The A* algorithm with an obstacle distance cost~\cite{kim_open-source_2022} (Baseline 1), (\romannumeral 2) the standard LQR-RRT* algorithm~\cite{perez_lqr-rrt_2012} (Baseline 2), (\romannumeral 3) our proposed method without the visibility constraint (Baseline 3), and (\romannumeral 4) the propsed Visibility-Aware RRT*. To highlight the effect of the visibility constraint, we illustrate the reference paths generated by Baseline 3 and our method in Fig.~\ref{fig:planner}. The results demonstrate that our method maintains fewer nodes in the tree compared to Baseline 3, as it only extends the nodes that satisfy the visibility constraint. This selective expansion leads to a more computationally efficient exploration of the configuration space. Furthermore, the final reference path generated by our method tends to take shallower turns around the known obstacles. This behavior complements the local tracking controller by providing a path that allows it to detect and avoid potential hidden obstacles. 

To demonstrate the effectiveness of the generated paths when interfacing with the local tracking controllers, we employ two types of controllers to track the paths as discussed in Sec.~\ref{sec:controller}: a CBF-QP controller and the Gatekeeper algorithm. We adopt the dynamic unicycle model for our controllers.

First, we evaluate the safety perspective of the reference path when it interfaces with the CBF-QP controller. We formulate the CBF-QP as in \cite{parwana_feasible_2023}, using the dynamic unicycle model, incorporating a CBF constraint for collision avoidance. In this experiment, we test two different FOV angles: $45^\circ$ and $70^\circ$. For Baseline 1, which is a deterministic planner, we run the algorithm once and for Baselines 2-3 and our proposed method, which are sampling-based planners, we generate 100 reference paths for each setting. After the planning phase, we place unknown obstacles near the known obstacles. Then, we evaluate whether the CBF-QP controller can successfully navigate the robot to the goal without collisions. One example of this experiment is shown in Fig.~\ref{fig:cbf-qp}. The results, summarized in Table~\ref{tab:controller}, demonstrate that even with a safety-critical controller using a valid CBF for collision avoidance, safety cannot be guaranteed if the planner does not consider the robot's perception limitations. This is because the robot detects the hidden obstacle too late, leaving no feasible solution for the CBF-QP to avoid the obstacle. This aspect becomes more prominent as the sensing capability becomes more limited. In contrast, the visibility-aware paths generated by our method can effectively complement the local tracking controller, ensuring safe navigation even in the presence of limited perception.

Next, we investigate the efficiency of the reference paths when interfacing with the Gatekeeper algorithm~\cite{agrawal_gatekeeper_2024}. Gatekeeper assesses the safety of the nominal trajectory and decides whether to follow it or execute a backup controller, which, in this case, is a stop command. If the next nominal trajectory is deemed unsafe, meaning the robot might not be able to completely stop within the known collision-free space given its control authority, Gatekeeper opts not to follow the nominal trajectory. Instead, it executes the backup controller and subsequently plans for global path again. We evaluate the percentage of the same set of reference paths that trigger the execution of the backup controller, indicating that the path was inefficient for navigation. One example of this experiment is shown in Fig.~\ref{fig:gatekeeper}. The results, presented in Table~\ref{tab:controller}, show that Visibility-Aware RRT* significantly reduces the need for replanning compared to the baseline planners. This demonstrates that the proposed method generates reference paths that are more efficient for navigation by minimizing the frequency of replanning required when used with controllers that treat unmeasured space as an unsafe set.

In addition to simulations, we conducted hardware experiments using a ground rover modeled as a dynamic unicycle (see Fig.~\ref{fig:main}). An RGB-D camera was used for localization and obstacle detection with a small safety buffer to mitigate sensor noise. All navigation tasks were operated onboard without a motion capture system. Hidden obstacles were not detected in time while tracking paths that were generated by baseline global path planners, leading to collisions. In Fig.~\ref{fig:main}, the naive path was generated by Baseline~3. In contrast, the paths generated by our method were able to avoid collisions by accounting for perception limitations. The detailed system setup and results are available in the supplementary video. 

%% file: _VII.Conclusion/conclusion.tex
In this paper, we considered that safety may not be guaranteed for robots with limited perception capabilities navigating in unknown environments, even when using safety-critical controllers. To address this challenge, we introduced the Visibility-Aware RRT* algorithm, which generates a global reference path that is both collision-free and visibility-aware, enabling the tracking controller to avoid hidden obstacles. We introduced the visibility CBF to encode the robot's visibility limitations, utilizing its constraint as a termination criterion during the steering process, alongside a standard collision avoidance CBF constraint. We also proved that our algorithm ensures the satisfaction of visibility constraints by the local controller, and we evaluated the performance of the integrated path planner with two different controllers via simulations. Our method outperformed all other compared baselines in both safety and efficiency aspects.